\documentclass[sigconf]{acmart}
\usepackage[ruled]{algorithm2e}
\usepackage{multirow}
\usepackage{float}
\usepackage{booktabs}
\usepackage{color}
\AtBeginDocument{%
  \providecommand\BibTeX{{%
    \normalfont B\kern-0.5em{\scshape i\kern-0.25em b}\kern-0.8em\TeX}}}

\copyrightyear{2021}
\acmYear{2021}
\setcopyright{acmcopyright}
\acmConference[SIGIR '21] {Proceedings of the 44th International ACM SIGIR Conference on Research and Development in Information Retrieval}{July 11--15, 2021}{Virtual Event, Canada.}
\acmBooktitle{Proceedings of the 44th International ACM SIGIR Conference on Research and Development in Information Retrieval (SIGIR '21), July 11--15, 2021, Virtual Event, Canada}
\acmPrice{15.00}
\acmISBN{978-1-4503-8037-9/21/07}
\acmDOI{10.1145/3404835.3462917}

\begin{document}
\fancyhead{}
\title{Enhanced Doubly Robust Learning for Debiasing Post-Click Conversion Rate Estimation}

\renewcommand{\thefootnote}{\fnsymbol{footnote}}
\author{Siyuan Guo$^{1,}$\footnotemark[1], Lixin Zou$^{2,}$\footnotemark[2], Yiding Liu$^2$, Wenwen Ye$^2$, Suqi Cheng$^2$, Shuaiqiang Wang$^2$, \\ Hechang Chen$^{1,4,}$\footnotemark[2], Dawei Yin$^2$ and Yi Chang$^{1,3,4,}$\footnotemark[2]}
\affiliation{$^1$School of Artificial Intelligence, Jilin University, $^2$Baidu Inc., $^3$International Center of Future Science, Jilin University
\country{}}
\affiliation{$^4$Key Laboratory of Symbolic Computation and Knowledge Engineering of Ministry of Education, Jilin University
\country{}}
\email{{guosyjlu,zoulixin15,liuyiding.tanh,yewenwen.hi,chengsuqi,shqiang.wang}@gmail.com}
\email{yindawei@acm.org,{chenhc,yichang}@jlu.edu.cn}

\renewcommand{\authors}{Siyuan Guo, Lixin Zou, Yiding Liu, Wenwen Ye, Suqi Cheng, Shuaiqiang Wang, Hechang Chen, Dawei Yin, Yi Chang}









\begin{abstract}
Post-click conversion, as a strong signal indicating the user preference, is salutary for building recommender systems. However, accurately estimating the post-click conversion rate (CVR) is challenging due to the selection bias, i.e., the observed clicked events usually happen on users' preferred items. Currently, most existing methods utilize counterfactual learning to debias recommender systems. Among them, the doubly robust (DR) estimator has achieved competitive performance by combining the error imputation based (EIB) estimator and the inverse propensity score (IPS) estimator in a doubly robust way. However, inaccurate error imputation may result in its higher variance than the IPS estimator. Worse still, existing methods typically use simple model-agnostic methods to estimate the imputation error, which are not sufficient to approximate the dynamically changing model-correlated target (i.e., the gradient direction of the prediction model). To solve these problems, we first derive the bias and variance of the DR estimator. Based on it, a more robust doubly robust (MRDR) estimator has been proposed to further reduce its variance while retaining its double robustness. Moreover, we propose a novel double learning approach for the MRDR estimator, which can convert the error imputation into the general CVR estimation. Besides, we empirically verify that the proposed learning scheme can further eliminate the high variance problem of the imputation learning. To evaluate its effectiveness, extensive experiments are conducted on a semi-synthetic dataset and two real-world datasets. The results demonstrate the superiority of the proposed approach over the state-of-the-art methods. The code is available at \url{https://github.com/guosyjlu/MRDR-DL}. 
\footnotetext[1]{Work performed during an internship at Baidu Inc.}
\footnotetext[2]{Joint corresponding authors.}

\end{abstract}




\maketitle
\renewcommand{\thefootnote}{\arabic{footnote}}
\section{Introduction}
E-commerce recommender systems aim at not only helping users explore the items of their interests, but also increasing revenues for the platform. Therefore, estimating the post-click conversion rate (CVR), i.e., the probability of an item being purchased after it is clicked, is a crucial task for building such systems in practice. Moreover, post-click conversion feedbacks have been widely recognized as strong signals for the learning of the recommender systems, as they explicitly express the user preference and directly contribute to the gross merchandise volume (GMV) of the platform \cite{gmv, dr-cvr}. However, it is very challenging to model such signals, which are extremely sparse in real-world applications. In particular, the post-click conversion feedbacks can only be observed in clicked events, which make up a tiny fraction of all possible user behaviors, while the other conversion feedbacks for unclicked events are missing. As such, a fundamental problem of CVR estimation is to study the missing mechanism in the post-click conversion feedbacks. 

For simplification, conventional CVR models usually assume that the missing conversion feedbacks are \emph{missing-at-random} (MAR). Such assumption can barely hold under the selection bias and recent studies \cite{rat,ips-implicit-learn,dr-ali} have shown that a recommendation model with MAR assumption often leads to sub-optimal results. On real-world e-commerce platforms, as users are free to click the items that they are likely to want to purchase (i.e., user self-selection), the observed clicked events are not representative samples of all the events, which makes the missing conversions \emph{missing-not-at-random} (MNAR). In other words, the fundamental reason behind the selection bias is that the users' propensities vary from item to item. Here, the propensity is defined as the probability of an item being clicked by a user, i.e., the click-through rate (CTR). Hence, in this paper, we adopt the MNAR assumption when estimating the post-click conversion rate, and focus on addressing the selection bias problem.

In recent years, three unbiased estimators in counterfactual learning have been applied to debiasing the CVR estimation. \textbf{(1) The error imputation based (EIB) estimator} \cite{pmf-debias,eib} computes an imputed error, i.e., the estimated value of the prediction error, for each unclicked event, and then uses it to estimate the true prediction error of all the events. However, this estimator often has a large bias due to the inaccurate error imputation, which will easily mislead the CVR estimation. \textbf{(2) The inverse propensity score (IPS) estimator} \cite{rat, gmcm, esmm} inversely weights the prediction error of each clicked event with its estimated CTR to correct the mismatch between the distributions of the clicked events and unclicked events. Although this estimator is unbiased given the ground-truth CTRs, it typically suffers from a high variance problem, which would lead to sub-optimal results. \textbf{(3) The doubly robust (DR) estimator} \cite{drjl, dr-ali, dr} combines the EIB estimator and IPS estimator to ensure both the low variance and low bias. Its unbiasedness is guaranteed if either the imputed errors or the CTRs are accurate. This property is called the double robustness.

Among the aforementioned estimators, the DR estimator has achieved initial success for debiasing recommender systems \cite{drjl,dr-ali,dr-cvr}. However, there are still two inherent challenges to be solved. Despite the double robustness, the DR estimator may increase the variance of the IPS estimator under inaccurate error imputation, which makes the learning process even complicated and leads to sub-optimal results. Hence, further variance reduction for the DR approaches deserves to investigate. Furthermore, although the DR estimator is more robust than the EIB and IPS estimator, it still requires relatively accurate CTR estimation and error imputation. In terms of the two tasks, the former has been extensively investigated by a lot of works \cite{ctr-deepfm,ctr-din}, whereas the latter rarely investigated. To estimate the imputed errors, previous DR based approaches typically introduce an extra imputation model that is agnostic of the prediction model, such as linear regression \cite{dr}, matrix factorization \cite{drjl}, multilayer perceptron (MLP) \cite{dr-ali}, etc. Here, the imputed errors, utilized as the gradient directions of the prediction model, should be dynamically changing during its learning process. However, simply using model-agnostic methods are not sufficient to approximate such a model-correlated target. Thus, it still calls for a better solution on how to model the error imputation.

To address the above-mentioned challenges, we propose the enhanced doubly robust learning approach for debiasing post-click conversion rate estimation. To tackle the first challenge, we propose to reduce the variance of the DR estimator, by redesigning the goal of the imputation learning ({\it i.e.}, the learning process of the imputation model) as the minimization of its variance \cite{mrdr, mrdr-init}. Specifically, we derive the bias and variance of the DR estimator, based on which we propose the more robust doubly robust (MRDR) estimator as a variant of the DR estimator to derive lower variance while retaining the double robustness. Moreover, inspired by Double DQN \cite{doubleDQN} in reinforcement learning, we propose a novel double learning approach for the MRDR estimator to tackle the second limitation. In particular, we adopt two CVR models with same structure but different parameters. The first one serves as the prediction model to learn from both the imputed errors and the true prediction errors for final CVR estimation. The second one serves as the imputation model to generate the pseudo label using its predicted CVR for each event. During the learning of the prediction model, the imputed error can be directly computed with the pseudo label and the predicted CVR. As such, we convert the error imputation into the general CVR estimation, and further, the imputed errors can be dynamically estimated in a model-correlated way. For the learning of both models, we alternate their learning process to enable them to be mutually regularized. In addition, we periodically update the parameters of the imputation model with the parameters of the prediction model, which is empirically beneficial for eliminating the high variance problem of the imputation learning. Extensive experiments are conducted on both semi-synthetic and real-world datasets to verify the effectiveness of both the proposed MRDR estimator and double learning approach.

The main contributions of this work are summarized as follows.
\begin{itemize}
    \item We conduct theoretical analysis on the bias and variance of the DR estimator, based on which we propose the more robust doubly robust (MRDR) estimator. It can achieve further variance reduction while retaining the double robustness.
    \item To dynamically utilize the information of the prediction model for error imputation, we propose a novel double learning approach for the MRDR estimator, which is also empirically beneficial for addressing the high variance problem of the imputation learning.
    \item Experimental results on the semi-synthetic dataset empirically verify the effectiveness of the proposed MRDR estimator. Furthermore, we conduct extensive experiments on two real-world datasets. The results show that the proposed enhanced doubly robust learning approach MRDR-DL outperforms the state-of-the-art methods.
\end{itemize}

\section{Preliminaries}
In this section, we detail the problem formulation, and introduce some existing unbiased estimators in the post-click conversion setting.

\subsection{Problem Formulation}
Let $\mathcal{U}=\{u_1, u_2, ..., u_m\}$ be the set of $m$ users, $\mathcal{I}=\{i_1, i_2, ..., i_n\}$ be the set of $n$ items, and $\mathcal{D}=\mathcal{U} \times \mathcal{I}$ be the set of all user-item pairs. We denote $\mathbf{R} \in \{0,1\}^{m \times n}$ as the conversion label matrix where each entry $r_{u,i}\in \{0,1\}$ indicates whether a conversion action occurs after user $u$ clicks item $i$. We use $\hat{\mathbf{R}} \in \mathbb{R}^{m \times n}$ to represent the predicted conversion rate matrix, where $ \hat{r}_{u,i} \in [0,1]$ represents the conversion rate predicted by a model. If we have a fully observed conversion label matrix $\mathbf{R}$, the ideal loss function for minimization can be formulated as
\begin{equation}
    \mathcal{L}_{ideal}( \hat{\mathbf{R}})=\frac{1}{|\mathcal{D}|}\sum_{(u,i)\in\mathcal{D}}e_{u,i},
\end{equation}
where $e_{u,i}$ is the prediction error. We usually adopt the cross entropy, $e_{u,i}=CE(r_{u,i}, \hat{r}_{u,i})=-r_{u,i}\log\hat{r}_{u,i}-(1-r_{u,i})\log(1-\hat{r}_{u,i})$ as the optimization goal for binary classification. Let $\mathbf{O}\in\{0,1\}^{m \times n}$ be the click indicator matrix with each entry $o_{u,i}=1$ if user $u$ clicks item $i$, and 0 otherwise. Since only post-click conversions for clicked events can be observed, the naive estimator estimates the ideal loss function by averaging the prediction error for clicked events as
\begin{equation}
\begin{aligned}
    \mathcal{L}_{naive}(\hat{\mathbf{R}})&=\frac{1}{|\mathcal{O}|}\sum_{(u,i)\in\mathcal{O}}e_{u,i}\\
    &=\frac{1}{|\mathcal{O}|}\sum_{(u,i)\in\mathcal{D}}o_{u,i}e_{u,i},
\end{aligned}
\end{equation}
where $\mathcal{O}=\{(u,i)|(u,i)\in \mathcal{D}, o_{u,i}=1\}$ denotes the clicked events. The naive estimator is intuitive and widely adopted by many existing methods. However, due to the selection bias, the conversions for unclicked events are MNAR, which leads to a biased estimation, i.e., $\mathbb{E}_{\mathbf{O}}[\mathcal{L}_{naive}(\hat{\mathbf{R}})]\neq\mathcal{L}_{ideal}( \hat{\mathbf{R}})$.
Previous works \cite{rat,drjl,ips-implicit-learn} have proved that the learning process based on a biased estimator often leads to a sub-optimal prediction model. Hence, it is essential to develop an unbiased estimator to address the MNAR problem. In the following, we will introduce three existing unbiased estimators.

\subsection{Error Imputation Based Estimator}
The error imputation based (EIB) estimator introduces an imputation model to compute the imputed error $\hat{e}_{u,i}$, i.e., the estimated value of the prediction error \cite{eib, pmf-debias}. Leveraging the imputed errors for unclicked events and the prediction errors for clicked events, we estimate the ideal loss function with the EIB estimator as
\begin{equation}
    \mathcal{L}_{EIB}( \hat{\mathbf{R}})=\frac{1}{|\mathcal{O}|}\sum_{(u,i)\in\mathcal{D}}o_{u,i}e_{u,i}+(1-o_{u,i})\hat{e}_{u,i}.
\end{equation}
When the imputed error $\hat{e}_{u,i}$ is accurate for any given unclicked event, the EIB estimator is unbiased, i.e., $\mathbb{E}_{\mathbf{O}}[\mathcal{L}_{EIB}(\hat{\mathbf{R}})]=\mathcal{L}_{ideal}( \hat{\mathbf{R}})$. However, the EIB estimator can hardly achieve accurate error imputation, and thus often has a large bias in practice, which would easily mislead the learning of the prediction model.

\subsection{Inverse Propensity Score Estimator}
The inverse propensity score (IPS) estimator \cite{rat, ips-implicit-learn, gmcm} weights each clicked event with $1/p_{u,i}$, where the propensity $p_{u,i}=\mathbb{P}(o_{u,i}=1)=\mathbb{E}[o_{u,i}]$ refers to the probability of the item being clicked by the user, i.e., the click-through rate (CTR) in the post-click conversion setting. By introducing an auxiliary CTR task to estimate the propensity $\hat{p}_{u,i}$, the IPS estimator can be formulated as
\begin{equation}
    \mathcal{L}_{IPS}( \hat{\mathbf{R}})=\frac{1}{|\mathcal{O}|}\sum_{(u,i)\in\mathcal{D}}\frac{o_{u,i}e_{u,i}}{\hat{p}_{u,i}}.
\end{equation}
The IPS estimator derives an unbiased estimate of the ideal loss function, i.e., $\mathbb{E}_{\mathbf{O}}[\mathcal{L}_{IPS}(\hat{\mathbf{R}})]=\mathcal{L}_{ideal}( \hat{\mathbf{R}})$, when the estimated propensity $\hat{p}_{u,i}$ is accurate for any given clicked event. However, as the clicked events merely account for a small part of $\mathcal{D}$, the CTR is typically assigned with a small value. Hence, the IPS estimator suffers from an especially severe high variance problem.

\subsection{Doubly Robust Estimator}
To address the large bias problem of the EIB estimator and the high variance problem of the IPS estimator, the doubly robust (DR) estimator is adopted by many previous works \cite{dr, drjl, dr-ali}. It combines the EIB estimator and the IPS estimator in a doubly robust way. Particularly, this estimator uses the imputed errors $\hat{e}_{u,i}$ to estimate the prediction errors for all the events, and correct the error deviation $\delta_{u,i}=e_{u,i}-\hat{e}_{u,i}$ for the unclicked events. The propensity $\hat{p}_{u,i}$ is inversely weighted to the error deviation for eliminating the MNAR effect. The loss function of the DR estimator can be defined as
\begin{equation}
    \mathcal{L}_{DR}( \hat{\mathbf{R}})=\frac{1}{|\mathcal{D}|}\sum_{(u,i)\in\mathcal{D}}\hat{e}_{u,i}+\frac{o_{u,i}(e_{u,i}-\hat{e}_{u,i})}{\hat{p}_{u,i}}.
\end{equation}
The DR estimator is unbiased, i.e., $\mathbb{E}_{\mathbf{O}}[\mathcal{L}_{DR}(\hat{\mathbf{R}})]=\mathcal{L}_{ideal}( \hat{\mathbf{R}})$, if either the imputed error $\hat{e}_{u,i}$ of any event or the propensity $\hat{p}_{u,i}$ of any clicked event is accurate. This property is recognized as double robustness. To compute the imputed errors, previous works typically introduce a separate imputation model. Since the imputation learning is actually a regression problem, DR uses the squared loss, 
\begin{equation}
    \mathcal{L}_{e}^{DR}=\sum_{(u,i)\in\mathcal{O}}\frac{(\hat{e}_{u,i}-e_{u,i})^2}{\hat{p}_{u,i}},
\end{equation}
to train the imputation model. The inverse propensity score is weighted to consider the MNAR effect, which also leads to the high variance problem of the imputation learning.

\section{Enhanced Doubly Robust Learning Approach}
\begin{figure*}
    \centering
    \includegraphics[width=\textwidth]{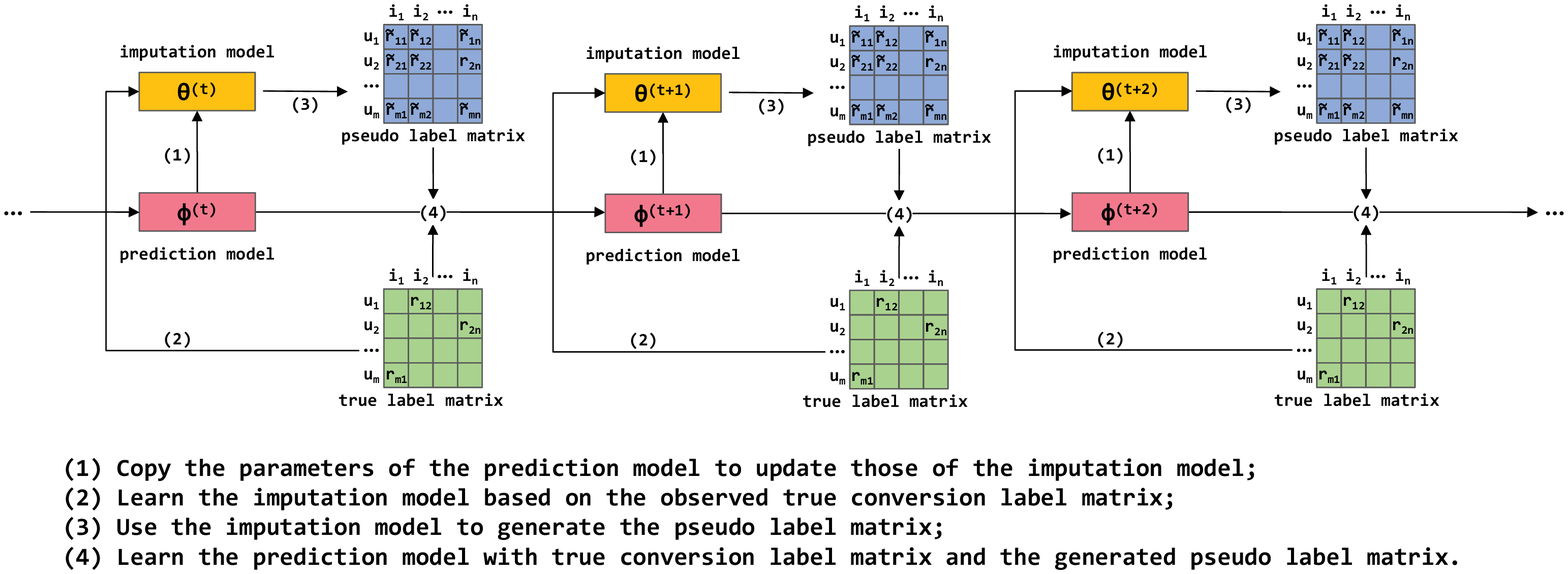}
    \caption{Workflow of the double learning approach.
}
    \label{fig:dl}
\end{figure*}
In this section, we elaborate the proposed enhanced doubly robust learning approach. We first analyze the bias and variance of the doubly robust estimator, based on which we propose the more robust doubly robust estimator for further variance reduction. Then, we detail the proposed novel double learning approach for the MRDR estimator.
\subsection{Bias and Variance Analysis of DR Estimator}
Initially, we formulate the bias of the DR estimator to prove its double robustness.
\begin{theorem}
Let $\delta_{u,i}=e_{u,i}-\hat{e}_{u,i}$ denote the additive error deviation, and $\Delta_{u,i}=1-\frac{p_{u,i}}{\hat{p}_{u,i}}$ the multiplicative propensity deviation. Then, the bias of the DR estimator is
\begin{equation}
    Bias\left[\mathcal{L}_{DR}(\hat{\mathbf{R}})\right]=\frac{1}{|\mathcal{D}|}\left|\sum_{(u,i)\in\mathcal{D}}\Delta_{u,i}\delta_{u,i}\right|.
\end{equation}
\label{bias}
\end{theorem}
\begin{proof}
    See Theorem 3.2 in \cite{dr-ali} for the proof.
\end{proof}
As shown in Theorem \ref{bias}, the DR estimator is close to the ideal loss function, i.e., $Bias\left[\mathcal{L}_{DR}(\hat{\mathbf{R}})\right]\approx0$, if either $\delta_{u,i}\approx0$ or $\Delta_{u,i}\approx0$, whereas the EIB estimator requires $\delta_{u,i}\approx0$ and the IPS estimator requires $\Delta_{u,i}\approx0$. This property is called double robustness. Then, we derive the variance of the DR estimator.
\begin{theorem}
    The variance of the DR estimator is
    \begin{equation}
        \mathbb{V}_{\mathbf{O}}[\mathcal{L}_{DR}(\hat{\mathbf{R}})]=\frac{1}{|\mathcal{D}|^2}\sum_{(u,i)\in\mathcal{D}}\frac{p_{u,i}(1-p_{u,i})}{\hat{p}_{u,i}^2}(e_{u,i}-\hat{e}_{u,i})^2.
    \end{equation}
    \label{variance}
\end{theorem}
\begin{proof}
    For a single term of the DR estimator, its variance on the click indicator $o_{u,i}$ is
    \begin{eqnarray}
        &&\mathbb{V}_{o_{u,i}}\left[\hat{e}_{u,i}+\frac{o_{u,i}(e_{u,i}-\hat{e}_{u,i})}{\hat{p}_{u,i}}\right]\nonumber\\
        &&=\mathbb{V}_{o_{u,i}}[o_{u,i}]\cdot\left(\frac{e_{u,i}-\hat{e}_{u,i}}{\hat{p}_{u,i}}\right)^2\nonumber\\
        &&=\left(\mathbb{E}_{o_{u,i}}[o_{u,i}^2]-\left(\mathbb{E}_{o_{u,i}}[o_{u,i}]\right)^2\right)\cdot\left(\frac{e_{u,i}-\hat{e}_{u,i}}{\hat{p}_{u,i}}\right)^2\nonumber\\
        &&=\mathbb{E}_{o_{u,i}}[o_{u,i}]\left(1-\mathbb{E}_{o_{u,i}}[o_{u,i}]\right)\cdot\left(\frac{e_{u,i}-\hat{e}_{u,i}}{\hat{p}_{u,i}}\right)^2\nonumber\\
        &&=\mathbb{E}_{o_{u,i}}\left[o_{u,i}\frac{1-p_{u,i}}{\hat{p}_{u,i}^2}(e_{u,i}-\hat{e}_{u,i})^2\right].
        \label{Vmrdr}
    \end{eqnarray}
    Then, summing across all terms of the DR estimator, we can derive the variance:
    \begin{equation}
        \begin{aligned}
        \mathbb{V}_{\mathbf{O}}[\mathcal{L}_{DR}(\hat{\mathbf{R}})]&=\mathbb{V}_{\mathbf{O}}\left[\frac{1}{|\mathcal{D}|}\sum_{(u,i)\in\mathcal{D}}\hat{e}_{u,i}+\frac{o_{u,i}(e_{u,i}-\hat{e}_{u,i})}{\hat{p}_{u,i}}\right]\\
        &=\frac{1}{|\mathcal{D}|^2}\sum_{(u,i)\in\mathcal{D}}\mathbb{V}_{o_{u,i}}\left[\hat{e}_{u,i}+\frac{o_{u,i}(e_{u,i}-\hat{e}_{u,i})}{\hat{p}_{u,i}}\right]\\
        &=\frac{1}{|\mathcal{D}|^2}\sum_{(u,i)\in\mathcal{D}}\frac{p_{u,i}(1-p_{u,i})}{\hat{p}_{u,i}^2}(e_{u,i}-\hat{e}_{u,i})^2.
        \end{aligned}
    \end{equation}
\end{proof}
Similarly, we can derive the variance of the IPS estimator as
\begin{equation}
    \mathbb{V}_{\mathbf{O}}[\mathcal{L}_{IPS}(\hat{\mathbf{R}})]=\frac{1}{|\mathcal{D}|^2}\sum_{(u,i)\in\mathcal{D}}\frac{p_{u,i}(1-p_{u,i})}{\hat{p}_{u,i}^2}e_{u,i}^2.
    \label{Vips}
\end{equation}
Theorem \ref{variance} and Equation \ref{Vips} illustrate that the variance of both estimators depends on the estimated propensity, i.e., the predicted CTR $\hat{p}_{u,i}$, which may lead to a high variance problem. However, it is worth noting that the DR estimator still reduces the variance of the IPS estimator, if any given event satisfies $0\le \hat{e}_{u,i}\le 2e_{u,i}$.

\subsection{More Robust Doubly Robust Estimator}
The theoretical analysis demonstrates that despite the double robustness, the DR estimator suffers from the risk of increasing the variance of the IPS estimator under inaccurate error imputation. Hence, we propose a more robust doubly robust (MRDR) estimator for further variance reduction. Specifically, we propose to learn the imputation model of the DR estimator by minimizing its variance. In other words, it is a variation of the DR estimator, and the only difference is that its loss function for imputation learning is derived from minimizing DR's variance. This means that the proposed MRDR estimator not only retains the double robustness, but also derives a lower variance than the original DR estimator. Based on Equation \ref{Vmrdr}, we take the expectation on $\mathcal{D}$ and estimate $p_{u,i}$ with $\hat{p}_{u,i}$ to derive the loss function of the imputation learning in the MRDR estimator as
\begin{equation}
    \begin{aligned}
    \mathcal{L}_{e}^{MRDR}&=\sum_{(u,i)\in\mathcal{D}}o_{u,i}\cdot\frac{1-\hat{p}_{u,i}}{\hat{p}_{u,i}^2}\cdot(\hat{e}_{u,i}-e_{u,i})^2\\
    &=\sum_{(u,i)\in\mathcal{O}}\frac{1-\hat{p}_{u,i}}{\hat{p}_{u,i}}\cdot\frac{(\hat{e}_{u,i}-e_{u,i})^2}{\hat{p}_{u,i}}.
    \end{aligned}
    \label{MRDR_imputation}
\end{equation}
Comparing the loss function of imputation learning in MRDR with that in DR, we note that MRDR changes the weights from $1/\hat{p}_{u,i}$ to $(1-\hat{p}_{u,i})/\hat{p}_{u,i}^2$, which has the property  
\begin{equation}
\left\{
    \begin{array}{lr}
    \frac{1}{\hat{p}_{u,i}}<\frac{1-\hat{p}_{u,i}}{\hat{p}_{u,i}^2}, &\text{if }\hat{p}_{u,i}<\frac{1}{2}\\
    \frac{1}{\hat{p}_{u,i}}>\frac{1-\hat{p}_{u,i}}{\hat{p}_{u,i}^2}, &\text{if }\hat{p}_{u,i}>\frac{1}{2}
    \end{array}
\right..
\end{equation}
As such, the MRDR estimator increases the penalty of the clicked events with low propensity, and decreases the penalty of the rest of the clicked events. In this way, the imputation model is learned better, which further enables MRDR to reduce the variance of the DR estimator. 

\subsection{Double Learning Approach}
In this subsection, we detail the proposed double learning approach for the MRDR estimator. Given a vector $\mathbf{x}_{u,i}$ encoding all the features of user $u$ and item $i$, previous works typically introduce two separate models: an imputation model $\hat{e}_{u,i}=f_{\theta}(\mathbf{x}_{u,i})$ estimates the imputed errors, and a prediction model $\hat{r}_{u,i}=g_{\phi}(\mathbf{x}_{u,i})$ learns from the imputed errors and true conversion labels to predict the CVR. Here, the imputation model is agnostic of the prediction model, and merely takes the user-item features $\mathbf{x}_{u,i}$ for error imputation. In other words, during the learning process of the prediction model, the imputed error cannot be dynamically estimated. From an optimization perspective, the imputation model plays the role of estimating the gradients for the learning of the prediction model. However, we argue that simply utilizing model-agnostic methods is not sufficient to approximate such a model-correlated target. To this end, we propose a novel double learning approach, which utilizes the pseudo-labelling technique to provide dynamically changing imputed errors for the prediction model. As such, the complicated error imputation is simplified as a general CVR estimation task. We show the workflow of the double learning approach in Figure \ref{fig:dl}. 

Specifically, we introduce two models with the same structure but different parameters: the prediction model $g_{\phi}(\mathbf{x}_{u,i})$ and the imputation model $g_{\theta}(\mathbf{x}_{u,i})$. When we need to learn the prediction model, we first generate the pseudo label $\Tilde{r}_{u,i}$ for each event based on the imputation model. Then, we estimate the imputed error by computing the cross entropy between the predicted conversion rate $\hat{r}_{u,i}$ and the pseudo label $\Tilde{r}_{u,i}$, i.e.,  $\hat{e}_{u,i}=CE(\Tilde{r}_{u,i}, \hat{r}_{u,i})$. In this way, the imputation model $g_{\theta}(\mathbf{x}_{u,i})$ is converted to a CVR estimation task; further to this, the original regression problem is converted to a binary classification problem. Therefore, we replace the squared loss term $(\hat{e}_{u,i}-e_{u,i})^2$ in Equation \ref{MRDR_imputation} with a cross-entropy term. The imputation learning process of the MRDR estimator is thereby redesigned as
\begin{equation}
    \mathcal{L}_{e}^{MRDR}(\theta)=\sum_{(u,i)\in\mathcal{O}}\frac{1-\hat{p}_{u,i}}{\hat{p}_{u,i}^2}\cdot CE(r_{u,i}, g_\theta(\mathbf{x}_{u,i}))+\lambda \|\theta\|^2_F,
    \label{imputation learning}
\end{equation}
where $\theta$ denotes all the parameters of the imputation model $g_{\theta}(\mathbf{x}_{u,i})$ and $\lambda$ controls the $L_2$ regularization strength to prevent overfitting. Note that, although we change the original formulation of the loss function of the imputation model in MRDR, the idea of increasing the penalty of the low-propensity clicked events and decreasing the penalty of the rest is kept. Meanwhile, we formulate the learning of the prediction model as
\begin{equation}
   \mathcal{L}_{r}^{MRDR}(\phi)=\sum_{(u,i)\in\mathcal{D}}\hat{e}_{u,i}+\frac{o_{u,i}(e_{u,i}-\hat{e}_{u,i})}{\hat{p}_{u,i}}+\mu\|\phi\|^2_F,
   \label{prediction learning}
\end{equation}
where $\phi$, $e_{u,i}=CE(r_{u,i}, g_{\phi}(\mathbf{x}_{u,i}))$ and $\hat{e}_{u,i}=CE(g_{\theta}(\mathbf{x}_{u,i}), g_{\phi}(\mathbf{x}_{u,i}))$ denote all the parameters of the prediction model $g_\phi(\mathbf{x}_{u,i})$, the prediction error, and the imputed error, and $\mu$ controls the $L_2$ regularization strength to prevent overfitting.

Inspired by Double DQN \cite{doubleDQN}, we redesign the learning approach of the both models. Generally, we alternate the learning process between the imputation model and the prediction model via mini-batch stochastic gradient descent. As such, two models regularize each other and jointly reach convergence. Since that the MRDR estimator merely enhances the inverse propensity weight of the imputation learning into $(1-\hat{p}_{u,i})/\hat{p}_{u,i}^2$, it suffers from the high variance of the imputation learning, which also happens in the DR estimator as mentioned in Section 2.4. Therefore, each time before learning the imputation model, we update its parameters with those of the prediction model, i.e., $\theta:=\phi$. In this way, the imputation model will be periodically corrected, and the information that the enhanced inverse propensity weight brings is kept. We empirically demonstrate that such learning scheme is beneficial for eliminating the high variance problem of the imputation learning. We summarize the proposed enhanced doubly robust learning approach, named MRDR-DL, in Algorithm 1.

\begin{algorithm}[t]
\caption{The Proposed Enhanced Doubly Robust Learning Approach, MRDR-DL}
\LinesNumbered
\KwIn{$\mathcal{O}$, $\mathcal{D}$, $\hat{p}$}
\KwOut{$\phi$}
Initialize the parameters $\theta$, $\phi$\\
\While{stopping criteria is not satisfied}{
    $\theta:=\phi$\\
	\For {number of steps for training the imputation model}{
	Sample a batch of clicked events from $\mathcal{O}$\\
	Update $\theta$ by descending along the gradient $\bigtriangledown_{\theta}\mathcal{L}_e^{MRDR}(\theta)$\\
	}
	Generate pseudo label $\Tilde{r}_{u,i}$ for any event $\forall (u,i)\in\mathcal{D}$\\
	\For {number of steps for training the prediction model}{
	Sample a batch of events from $\mathcal{D}$\footnote{}\\
	Update $\phi$ by descending along the gradient $\bigtriangledown_{\phi}\mathcal{L}_r^{MRDR}(\phi)$\\
	}
}
\end{algorithm}
\footnotetext[1]{Due to the sparsity of the clicked events, we decrease the sample probability of the unclicked events in practice.}
\section{Semi-synthetic Experiments}
Following previous works \cite{rat, drjl, dr-cvr}, we conduct semi-synthetic experiments to investigate the following research question (RQ).
\begin{enumerate}
    \item[\textbf{RQ1.}] Does the MRDR estimator lead to more accurate loss estimation than other estimators?
\end{enumerate}
\subsection{Experimental Setup}
\subsubsection{Dataset and preprocessing} 
\begin{table}[t]
\begin{tabular}{lccc}
\toprule
 & ML 100K & Coat Shopping & Yahoo! R3 \\
 \midrule
\#users & 943 & 290 & 15400 \\
\#items & 1682 & 300 & 1000 \\
\#MNAR ratings & 100000 & 6960 & 311704 \\
\#MAR ratings & 0 & 4640 & 54000 \\
\bottomrule
\end{tabular}
\caption{Statistic details of the datasets.}
\label{statistic}
\end{table}
To compute the accuracy of the estimated loss, we need a fully observed conversion label matrix, which is unavailable in real-world datasets. Thus, we create a semi-synthetic evaluation dataset using the MovieLens (ML) 100K\footnote{\url{https://grouplens.org/datasets/movielens/}} \cite{movielens} dataset in order to allow us to conduct the semi-synthetic experiment. The statistical details of the dataset are presented in Table \ref{statistic}. We employ the following preprocessing procedures \cite{dr-cvr} to convert the explicit feedback setting to the post-click conversion setting, and derive a fully observed conversion label matrix and a click indicator matrix.

(1) Use matrix factorization \cite{mf} to complete the rating matrix, but the predicted ratings are unrealistically high for all user-item pairs. To match a more realistic rating distribution $[p_1, p_2, p_3, p_4, p_5]$ given in \cite{ratio}, we sort all the ratings in ascending order, assign a value of 1 to the bottom $p_1$ fraction of the matrix entries, assign a value of 2 to the next $p_2$ fraction, and so on.

(2) Transform the predicted ratings $R_{u,i}\in\{1,2,3,4,5\}$ into CTR $p_{u,i}\in(0,1)$ with $p_{u,i}=p\alpha^{\min(4, 6-R_{u,i})}$, where $p$ is set to 1 and $\alpha$ is set to 0.5 in our experiments.

(3) Transform the predicted ratings $R_{u,i}$ into true CVR $r_{u,i}^{true}$ by correspondingly replacing the rating $\{1, 2, 3, 4, 5\}$ with the conversion rate $\{0.1, 0.3, 0.5, 0.7, 0.9\}$. Note that we can only observe the binary conversion labels rather than the true values of the CVR in practice. Thus, we simply assign fixed values to them based on different predicted ratings. 

(4) Sample the binary click indicator and conversion label with the Bernoulli sampling; that is, 
\begin{equation}
    o_{u,i}\sim Bern(p_{u,i}), r_{u,i}\sim Bern(r_{u,i}^{true}),\forall(u,i)\in\mathcal{D},
    \label{sample}
\end{equation}
where $Bern(\cdot)$ denotes the Bernoulli distribution. Thereafter, we can derive a fully-observed conversion label matrix $\mathbf{R}$ and a click indicator matrix $\mathbf{O}$. 

\subsubsection{Experimental details}
Given a predicted CVR matrix $\hat{\mathbf{R}}$, we can directly compute the ideal loss by averaging the prediction error of each entry between $\mathbf{R}$ and $\hat{\mathbf{R}}$. In contrast, the estimators derive the estimated loss with partial entries in $\mathbf{R}$ whose corresponding click indicators $o_{u,i}$ equal 1. To evaluate the performance of loss estimation, we use the following five predicted CVR matrices \cite{rat,drjl} for comparison.
\begin{itemize}
    \item \textbf{ONE}: The predicted conversion rate $\hat{r}_{u,i}$ is identical to the true CVR $r_{u,i}^{true}$, except that $|\{(u,i)|r_{u,i}^{true}=0.9\}|$ randomly selected true CVR of 0.1 are flipped to 0.9.
    \item \textbf{THREE}: Same as ONE, but flipping true CVR of 0.3 instead.
    \item \textbf{FIVE}: Same as ONE, but flipping true CVR of 0.5 instead.
    \item \textbf{SKEW}: The predicted conversion rate is sampled from the Gaussian distribution $\mathcal{N}(\mu=r_{u,i}^{true}, \sigma=\frac{1-r_{u,i}^{true}}{2})$, and clipped to $[0.1, 0.9]$.
    \item \textbf{CRS}: The predicted conversion rate $\hat{r}_{u,i}=0.1$ if the true CVR $r_{u,i}^{true}\le0.7$.  Otherwise, $\hat{r}_{u,i}=0.5$.
\end{itemize}
We compare the MRDR estimator with the naive, EIB, IPS, and DR estimators. We estimate the propensity as $\frac{1}{\hat{p}_{u,i}}=\frac{1-\beta}{p_{u,i}}+\frac{\beta}{p_{e}}$, where $p_e=\frac{1}{|\mathcal{D}|}\sum_{(u,i)\in\mathcal{D}}o_{u,i}$, and $\beta$ is set to 0.5 to introduce noises.  For EIB and DR, the imputed error is computed as $\hat{e}_{u,i}=CE(\frac{\sum_{(u,i)\in\mathcal{D}}r_{u,i}/\hat{p}_{u,i}}{\sum_{(u,i)\in\mathcal{D}}1/\hat{p}_{u,i}},\hat{r}_{u,i})$. For MRDR, we compute the imputed errors as $\hat{e}_{u,i}=CE(\frac{\sum_{(u,i)\in\mathcal{D}}(1-\hat{p}_{u,i})r_{u,i}/\hat{p}_{u,i}^2}{\sum_{(u,i)\in\mathcal{D}}1/\hat{p}_{u,i}},\hat{r}_{u,i})$.
\subsubsection{Evaluation metric}
We compare the performance of the five estimators by the relative error (RE) as
\begin{equation}
    \text{RE}(\mathcal{L}_{estimator})=\frac{\left|\mathcal{L}_{ideal}(\hat{\mathbf{R}})-\mathcal{L}_{estimator}(\hat{\mathbf{R}})\right|}{\mathcal{L}_{ideal}(\hat{\mathbf{R}})},
\end{equation}
where $\mathcal{L}_{estimator}$ denotes the estimator to be compared. RE evaluates the accuracy of the estimated loss, and a smaller value of MRE means a higher accuracy.
\subsection{Experiment Results (RQ1)}
\begin{table}[t]
\begin{tabular}{lccccc}
\toprule
      & naive  & EIB    & IPS    & DR     & MRDR            \\ \midrule
ONE   & 0.0686 & 0.5427 & 0.0346 & 0.0131 & \textbf{0.0073} \\
THREE & 0.0792 & 0.5869 & 0.0401 & 0.0172 & \textbf{0.0047} \\
FIVE  & 0.1023 & 0.6152 & 0.0515 & 0.0138 & \textbf{0.0119} \\
SKEW  & 0.0255 & 0.3574 & 0.0124 & 0.0081 & \textbf{0.0013} \\
CRS   & 0.1773 & 0.0610 & 0.0888 & 0.0551 & \textbf{0.0503} \\
\bottomrule
\end{tabular}
\caption{RE of the five estimators compared to the ideal loss.}
\label{synthetic}
\end{table}
In Table \ref{synthetic}, we report the averaged RE of the five estimators over 20 times of sampling with Equation \ref{sample}. We can see that IPS, DR and MRDR estimators outperform the naive estimator in every setting. This is caused by the selection bias that we introduce by controlling the propensity $p_{u,i}$. In contrast, the EIB estimator derives the worst RE in four settings; this is mainly due to the large bias of the heuristic error imputation. Additionally, the DR estimator improves the performance of the IPS estimator by jointly considering the imputed errors and the estimated propensities. Over all the settings, the MRDR estimator achieves the best performance, which can be attributed to both the double robustness and the reduced variance. Overall, the results conclude that our proposed method can achieve more accurate loss estimation. Next, we further evaluate our method on the task of CVR estimation on real-world datasets.

\section{Real-world Experiments}
\begin{table*}
\begin{tabular}{ccccccccc}
\hline
 & \multicolumn{1}{l}{} & \multicolumn{3}{c}{DCG@K} &  & \multicolumn{3}{c}{Recall@K} \\ \cline{3-5} \cline{7-9} 
\multicolumn{1}{c}{Datasets} & Methods & K=2 & K=4 & K=6 &  & K=2 & K=4 & K=6 \\ \hline
\multicolumn{1}{c}{\multirow{4}{*}{Coat Shopping}} & naive & 0.6694$\pm$0.0136  & 0.9432$\pm$0.0138 & 1.1321$\pm$0.0126 &  & 0.8054$\pm$0.0159 & 1.3903$\pm$0.0225 & 1.8991$\pm$0.0233 \\
\multicolumn{1}{c}{} & IPS &0.7093$\pm$0.0232  &0.9552$\pm$0.0223  &1.1248$\pm$0.0217  &  &0.8249$\pm$0.0298  &1.3520$\pm$0.0353  &1.8078$\pm$0.0399\\
\multicolumn{1}{c}{} & DR-JL &0.6771$\pm$0.0273  &0.9266$\pm$0.0282  &1.0962$\pm$0.0272  &  &0.7949$\pm$0.0337  &1.3286$\pm$0.0420  &1.7849$\pm$0.0456 \\
\multicolumn{1}{c}{} & MRDR-DL &\textbf{0.7219}$\pm$0.0211  &\textbf{0.9905}$\pm$0.0204  &\textbf{1.1696}$\pm$0.0217  &  &\textbf{0.8499}$\pm$0.0265  &\textbf{1.4249}$\pm$0.0321  &\textbf{1.9060}$\pm$0.0430 \\ \hline \hline
\multirow{4}{*}{Yahoo! R3} & Naive &0.5469$\pm$0.0058  &0.7466$\pm$0.0049  &0.8714$\pm$0.0040  &  &0.6479$\pm$0.0066  &1.0745$\pm$0.0074  &1.4098$\pm$0.0062\\
 & IPS &0.5502$\pm$0.0018  &0.7520$\pm$0.0018  &0.8751$\pm$0.0014  &  &0.6545$\pm$0.0021  &1.0797$\pm$0.0025  &1.4168$\pm$0.0025\\
 & DR-JL &0.5310$\pm$0.0045  &0.7273$\pm$0.0053  &0.8512$\pm$0.0045  &  &0.6292$\pm$0.0049  &1.0495$\pm$0.0082  &1.3822$\pm$0.0081\\
 & MRDR-DL &\textbf{0.5561}$\pm$0.0058  &\textbf{0.7549}$\pm$0.0023  &\textbf{0.8811}$\pm$0.0036  &  &\textbf{0.6595}$\pm$0.0074  &\textbf{1.0846}$\pm$0.0054  &\textbf{1.4237}$\pm$0.0059\\ \hline
\end{tabular}
\caption{A comparison of the overall performance of MRDR-DL and competing methods on two real-world datasets.}
\label{overall}
\end{table*}
In this section, we compare the proposed learning approach with other existing debiasing approaches using real-world datasets. We anticipate the experimental results to answer the following RQs.
\begin{enumerate}
    \item[\textbf{RQ2.}] Does the proposed approach MRDR-DL lead to higher debiasing performance than existing approaches?
    \item[\textbf{RQ3.}] What influence do the various designs have on the proposed approach MRDR-DL?
    \item[\textbf{RQ4.}] How does the sample ratio of unclicked events to clicked events influence the performance of MRDR-DL?
    \item[\textbf{RQ5.}] How does the proposed double learning approach work for both the imputation model and the prediction model? 
\end{enumerate}
\subsection{Experimental Setup}
\subsubsection{Datasets and preprocessing}
To evaluate the performance of the unbiased CVR estimation, we need an MAR test set. However, as stated in \cite{dr-ali}, we cannot force users to randomly click items in order to generate unbiased data for CVR estimations. Previous work \cite{dr-cvr} simulates the unbiased CVR estimation setting by using the datasets with specific properties. First, the datasets need to contain explicit feedback, which can reveal ground-truth user preference information. Next, the datasets need to contain an additional MAR test set, where users are asked to rate randomly selected sets of items. This enables us to evaluate the performance of the unbiased CVR estimation. To the best of our knowledge, there are only two publicly available datasets that satisfy these requirements, i.e., Coat Shopping\footnote{\url{https://www.cs.cornell.edu/~schnabts/mnar}} and Yahoo! R3\footnote{\url{http://webscope.sandbox.yahoo.com/}}. The statistical details for both datasets are presented in Table \ref{statistic}.

For both the MNAR data and the MAR data of both datasets, we follow \cite{dr-cvr} and employ the following preprocessing procedure.

(1) We define the binary click indicator as $o_{u,i}=1$ if the item $i$ is rated by user $u$, and $o_{u,i}=0$ otherwise.

(2) We define the binary conversion label as $r_{u,i}=1$ if the item $i$ is rated greater than or equal to 4 by user $u$, and $r_{u,i}=0$ otherwise.

(3) We derive the post-click conversion dataset as $\{(u,i,r_{u,i})|o_{u,i}=1,\forall (u,i)\in\mathcal{D}\}$.

For both datasets, we randomly split the MNAR datasets into training (90\%) and validation (10\%) sets, while the MAR datasets are kept as the test sets. Following the previous works \cite{ips-implicit, dr-cvr}, we filter out users who have no conversion records in the test set.

\subsubsection{Baselines}
We compare the proposed method with the following baselines:
\begin{itemize}
    \item \textbf{Naive}: It simply uses the naive estimator as the loss function to estimate CVR.
    \item \textbf{IPS} \cite{rat}: It derives the IPS estimator as the loss function by estimating the CTR as the propensity score.
    \item \textbf{DR-JL} \cite{drjl}: It utilizes the DR estimator by jointly learning the imputation model and prediction model.
\end{itemize}
Due to the high bias problem, the EIB estimator is widely recognized as a weak baseline \cite{drjl,dr-cvr,rat}, and thus is not included in our comparison. In our experiments, both the CTR and the CVR are estimated by factorization machine (FM) \cite{fm}.

\subsubsection{Experimental Protocols}
We adopt the mini-batch Adam to optimize all the methods, with the default learning rate set at 0.001. We fix the mini-batch size to 1024 for both datasets. In terms of FM, the embedding size is fixed as 64. We tune the $L_2$ regularization coefficient $\lambda$ in the range of $\{1e^{-5}, 1e^{-4}, ..., 1\}$. Note that, for DR based methods, we apply a grid search when tuning the $L_2$ regularization coefficient of the imputation model and the prediction model; also, the sample ratio for unclicked events to clicked events is tuned in the range of $\{2, 4, 6, 8\}$. For CTR estimation, we fix the negative sampling ratio to 4.

For all the methods, we first choose the best hyper-parameters based on the validation set. Then, we perform the early stopping strategy (which applies if the model performance does not improve for five epochs) and report the best test result from the best-performing model on the validation set.

We use recall and discounted cumulative gain (DCG) to evaluate the debiasing performance of all the methods. We calculate both metrics for each user in the test set and report the average score.

\subsection{Overall Performance (RQ2)}
Table \ref{overall} shows the overall performance in terms of DCG@K and Recall@K ($K\in\{2,4,6\}$) on two real-world datasets. To reduce the effect of randomness, we repeat the experiments 100 times for Coat Shopping and 20 times for Yahoo! R3, and then report the mean and standard deviation for each. The best results are highlighted in boldface. From the table, we can see that the debiasing methods, IPS and MRDR-DL, outperform Naive for both datasets, demonstrating the necessity of handling the selection bias in the CVR estimation. Meanwhile, we find that although DR-JL utilizes the unbiased DR estimator, it still gives the worst performance on both datasets. One possible explanation for this is that DR-JL is originally designed for debiasing explicit MNAR feedbacks; as such, its joint learning approach may not be applicable to CVR estimation. Overall, the proposed method MRDR-DL consistently outperforms other methods on both datasets, which verifies the effectiveness of both the proposed MRDR estimator and the double learning approach.

\subsection{Ablation Study (RQ3)}
\begin{table*}
\begin{tabular}{ccccccccccccc}
\hline
 &  & \multicolumn{5}{c}{DCG@K} &  & \multicolumn{5}{c}{Recall@K} \\ \cline{3-7} \cline{9-13} 
Datasets & Methods & K=2 & K=3 & K=4 & K=5 & K=6 &  & K=2 & K=3 & K=4 & K=5 & K=6 \\ \hline
\multirow{4}{*}{Coat Shopping} & MRDR-DL & 0.7219 & 0.8728 & 0.9905 & 1.0878 & 1.1695 &  & 0.8499 & 1.1518 & 1.4249 & 1.6765 & 1.9060 \\
 & DR-DL & 0.7205 & 0.8670 & 0.9806 & 1.0778 & 1.1601 &  & 0.8438 & 1.1368 & 1.4004 & 1.6517 & 1.8827 \\
 & MRDR-JL & 0.6948 & 0.8442 & 0.9613 & 1.0582 & 1.1442 &  & 0.8227 & 1.1215 & 1.3935 & 1.6439 & 1.8852 \\
 & MRDR-DL with SL & \textbf{0.7255} & 0.8720 & 0.9871 & 1.0827 & 1.1651 &  & \textbf{0.8504} & 1.1434 & 1.4107 & 1.6580 & 1.8892 \\ \hline \hline
\multirow{4}{*}{Yahoo! R3} & MRDR-DL & 0.5561 & 0.6694 & 0.7549 & 0.8234 & 0.8811 &  & 0.6595 & 0.8860 & 1.0846 & 1.2616 & 1.4237 \\
 & DR-DL & 0.5463 & 0.6602 & 0.7459 & 0.8145 & 0.8714 &  & 0.6484 & 0.8762 & 1.0752 & 1.2525 & 1.4123 \\
 & MRDR-JL & 0.5546 & 0.6668 & 0.7544 & 0.8221 & 0.8786 &  & 0.6584 & 0.8828 & \textbf{1.0862} & 1.2612 & 1.4199 \\
 & MRDR-DL with SL & 0.5321 & 0.6439 & 0.7287 & 0.7963 & 0.8538 &  & 0.6298 & 0.8535 & 1.0503 & 1.2251 & 1.3863 \\ \hline
\end{tabular}
\caption{Ablation study of MRDR-DL on two real-world datasets.}
\label{as}
\end{table*}

\begin{figure*}
    \centering
    \includegraphics[width=\textwidth]{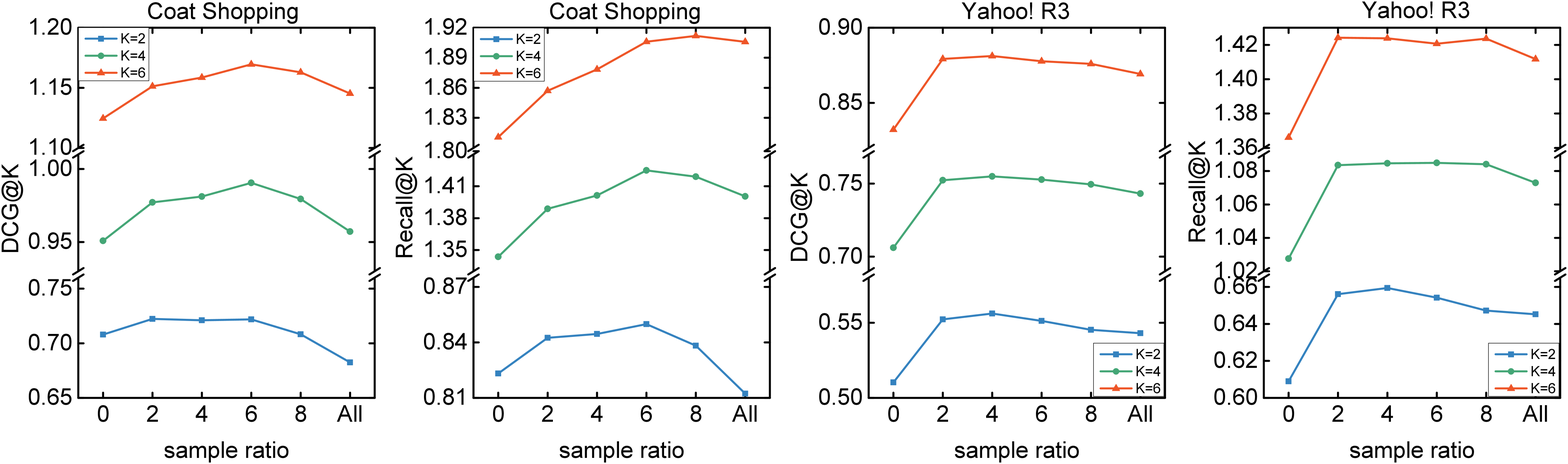}
    \caption{Effect of the sample ratio of unclicked events to clicked events. "All" means that we sample from all the events.}
    \label{ps}
\end{figure*}

To apply the DR estimator to the post-click conversion setting, the proposed method, MRDR-DL, has specific design features. In this section, we will analyze their respective impacts on the method's performance via an ablation study. The experimental results for MRDR-DL and its three variants on two datasets are summarized in Table \ref{as}. The results that are better than MRDR-DL are highlighted in boldface. We detail the variants and analyze their respective effects as follows.

(1) DR-DL: We replace the MRDR estimator with the DR estimator, i.e., we change the weights of the imputation learning from $(1-\hat{p}_{u,i})/\hat{p}_{u,i}^2$ to $1/\hat{p}_{u,i}$. The results imply that enhancing the weights to adjust the penalty for the clicked events based on varying propensities is conducive to the variance reduction of the DR estimator, and further improves the performance of the prediction model.

(2) MRDR-JL: We alternate the training of the imputation model and the prediction model without sharing the parameters periodically (i.e., we skip Step 3 of the Algorithm 1). The experiment results verify the necessity of periodically correcting the imputation model based on the prediction model, which is empirically beneficial for eliminating the high variance problem of the imputation learning. 

(3) MRDR-DL with SL: We replace the cross-entropy term of the imputation learning with the squared loss term, which is theoretically derived from the variance of the DR estimator. The results of the variant are consistent with MRDR-DL on Coat Shopping and significantly better than MRDR-DL on Yahoo! R3. One reason is that the squared loss aims at minimizing the deviation between imputed errors and true prediction errors, whereas the pseudo-label generation is essentially a binary classification problem. Hence, it is more intuitive to have cross entropy as the optimization goal.

\subsection{Parameter Sensitivity Study (RQ4)}
By jointly considering both clicked and unclicked events, DR based estimators can enjoy the double robustness. To investigate the impact of the unclicked events for the proposed MRDR-DL method, we vary the sample ratio of unclicked events to clicked events in the range of \{0, 2, 4, 6, 8, All\}. Here, "All" means that the sample ratio is set to the maximum possible value, which is 12.5 for Coat Shopping and 49.4 for Yahoo! R3. Figure \ref{ps} shows DCG@K and Recall@K for MRDR-DL with respect to different sample ratios on both datasets. As we can see, MRDR-DL with a sample ratio of 0 (i.e., we merely sample from the clicked events) derives the worst performance in most settings. This shows that the well-learned imputation model enables the unclicked events to provide the prediction model with useful information. Furthermore, we find that sampling from all the events adversely hurts the performance of the prediction model, even though we should in theory. One reason for this might be that clicked events are typically sparse in the real-world datasets, meaning that we cannot ensure that the prediction model obtains sufficient information. For both datasets, the optimal sample ratio is around 4 to 8. Setting the sample ratio too conservatively or too aggressively may adversely affect the prediction performance.

\subsection{Analysis of the Double Learning Approach (RQ5)}
In this subsection, we will further investigate the proposed double learning approach. We plot the training curves of the prediction model and the imputation model of MRDR-DL on Coat Shopping in Figure \ref{curve}. In the proposed method MRDR-DL, the prediction model aims at estimating CVR, while the imputation model aims at computing the imputed errors. Hence, we adopt DCG@4, and mean absolute error (MAE) for imputed errors and true prediction errors, respectively, to evaluate their testing performance. As shown in Figure \ref{ps}, the training loss of the prediction model slightly fluctuates in the first 300 epochs before gradually reaching convergence. In contrast, the training curve of the imputation model is more stable. The reason for this is that the imputation model is not well-trained enough to provide the prediction model with sufficiently accurate information at the very beginning, whereas the imputation model is trained on clicked events with ground-truth labels. Further epochs of the double learning approach enable both models to exchange their information periodically. In this way, both models are jointly well-learned, reaching convergence together after about 900 epochs. Note that the testing curve of the prediction model fluctuates in the training process. This is reasonable because we train it with point-wise loss (i.e., cross entropy), whereas we evaluate it with a list-wise metric (i.e., DCG@4) when checking its debiasing performance.
\begin{figure}[t]
    \centering
    \includegraphics[width=\linewidth]{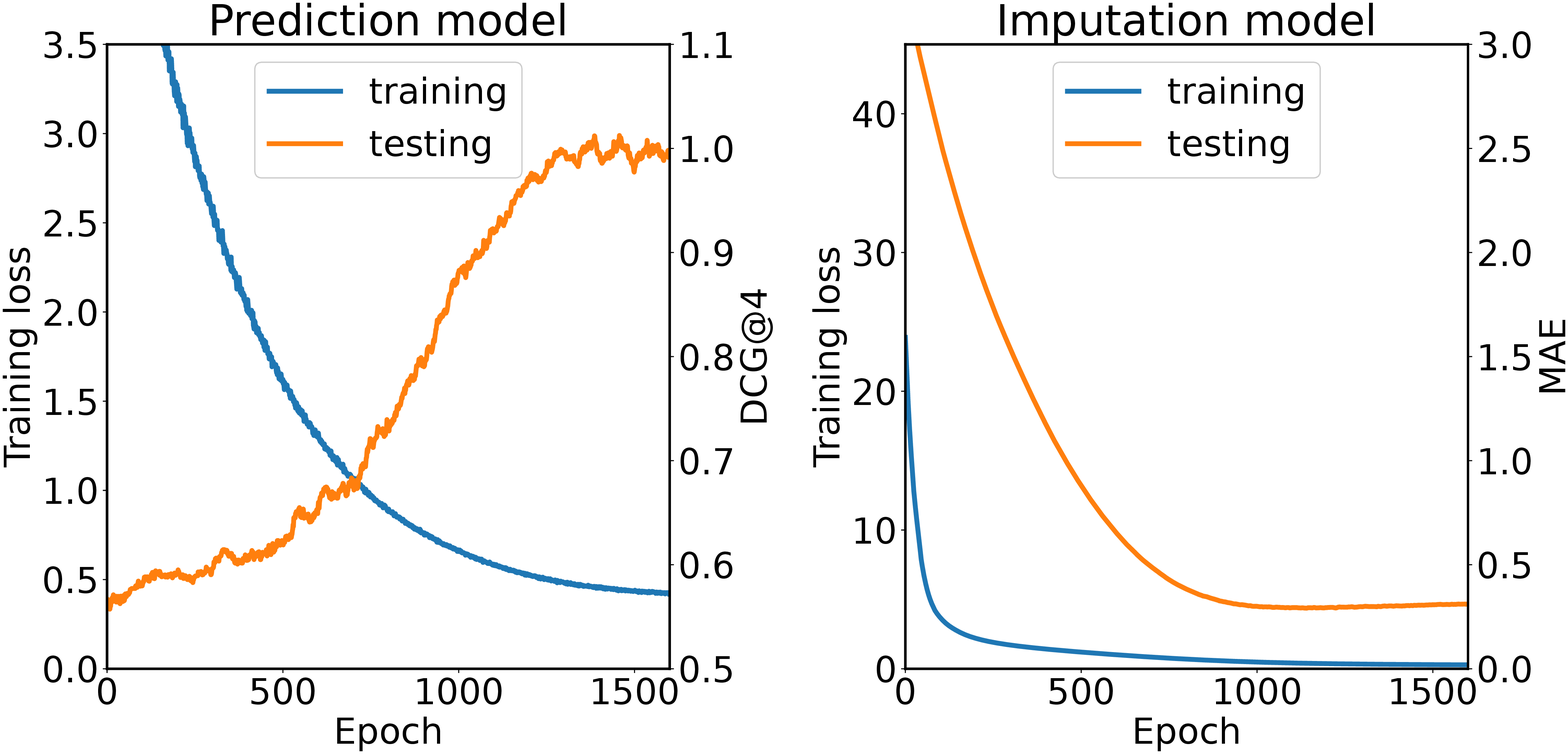}
    \caption{Training curves for the prediction model and the imputation model of MRDR-DL on Coat Shopping.}
    \label{curve}
\end{figure}
\section{Related Work}
\subsection{General Approaches to CVR Estimation}
CVR estimation is a key component of the recommender system because it directly contributes to the final revenue. Due to the inherent similarity, CVR estimation typically refers to the advances made by the CTR prediction task and implicit recommendation in practice, such as traditional models \cite{yd-poi,ctr-fm}, deep learning based models \cite{ctr-deepfm, ctr-din, deepctr-hfm, ww-dl} and reinforcement learning based models \cite{lx-rl1,lx-rl2,lx-rl3}. However, few studies directly investigate the CVR estimation tasks. Previous works often employ traditional models such as logistic regression \cite{cvr-lr1, cvr-lr2} and gradient boosting decision tree \cite{gbdt}, while deep learning techniques like neural network \cite{esmm, esm2} and graph convolution network \cite{gmcm,yd-gcn} are also adopted for CVR estimation. However, the selection bias issue is still underexplored, which has a significant influence on improving performance in practice.

\subsection{Counterfactual Learning from MNAR Data}
Most data for learning the recommender systems are MNAR, which is caused by various biases, including selection bias, conformity bias, exposure bias, etc \cite{survey}. Previous works typically adopt counterfactual learning methods to address these issues. Specifically, EIB methods \cite{pmf-debias, eib} rely on a missing data model to model the missing mechanism. IPS methods employ logistic regression \cite{rat}, expectation-maximization algorithm \cite{wsdm21}, and matrix completion \cite{1bitmc} to estimate the propensity for correcting the mismatch between observed and unobserved data. DR methods \cite{dr-ali,drjl} utilize an imputation model and a prediction model to jointly learn from the MNAR data. Other methods based on information bottleneck \cite{cvib}, meta learning \cite{at}, and causal embedding \cite{cause,dice} have been also explored to address these biases. Among above, IPS and DR has been widely applied to the recommender systems. However, how to specify appropriate error imputation and propensity estimation is a critical issue affecting their unbiasedness, which needs to be resolved in the post-click conversion setting. 

\subsection{Selection Bias in CVR Estimation}
Selection bias is ubiquitous in recommender systems, especially in the CVR estimation task. A few studies have investigated it, achieving effective results. ESMM \cite{esmm} models both the CTR and CVR tasks, using muti-task learning to eliminate the selection bias issue in a heuristic way. Similarly, $ESM^2$ \cite{esm2}, which is also essentially biased, extends ESMM by introducing additional auxiliary tasks. In contrast, GMCM \cite{gmcm} uses the IPS estimator to derive unbiased error evaluation when learning the CVR estimation task. In addition, Multi-IPW and Multi-DR \cite{dr-ali} enjoy the unbiasedness of the IPS and DR estimator by learning both CTR and CVR tasks through multi-task learning. Although they consider the selection bias, the above approaches are evaluated using biased datasets; thus, their experimental results cannot be used to verify their debiasing performance, which is a widely-recognized limitation in practice. Furthermore, a recent work \cite{dr-cvr} proposes to utilize the DR estimator for debiasing ranking metric with post-click conversions, which mainly concerns the evaluation of the recommender systems. In contrast, we focus on debiasing the learning of the CVR estimation, and we use two real-world datasets containing unbiased data to evaluate the debiasing performance.

\section{Conclusion and Future Work}
In this paper, we explore the problem of the selection bias in post-click conversion rate estimation. First, we analyze the bias and the variance of the DR estimator. Then, based on the theoretical analysis, we propose the more robust doubly robust estimator, which reduces the variance of the DR estimator while retaining the double robustness. Finally, we propose a novel double learning approach for MRDR estimator. It can dynamically utilize the information of the prediction model for the imputation model and empirically eliminate the high variance problem of the imputation learning. In the experiments, we verify the effectiveness of the proposed MRDR estimator on semi-synthetic datasets. In addition, we conduct extensive experiments on two real-world datasets to demonstrate the superiority of the proposed debiasing approach. For future work, we believe that the explainability \cite{taert} of the debiasing approach warrants further investigation.

\begin{acks}
This work is supported by National Natural Science Foundation of China (No.61976102, No.U19A2065 and No.61902145).

\end{acks}
\bibliographystyle{plain}
\bibliography{ref.bib}
\end{document}